\documentclass[a4paper,twocolumn]{article}          
\usepackage[left=1.5cm, right=1.5cm]{geometry}

\usepackage[utf8]{inputenc}
\usepackage[english]{babel}

\usepackage{bm}
\usepackage{bbm}
\usepackage{amsthm}
\usepackage{amssymb} 
\usepackage{amsmath}
\usepackage[all,color]{xy}
\usepackage[round]{natbib}
\usepackage{algorithm}
\usepackage{algorithmic}

\def\diag{\mathrm{diag}}

\newcommand{\E}{{\rm E}\,}
\newcommand{\sign}{{\rm sign}}
\newcommand{\Wish}{{\cal W}}

\newcommand{\tr}{{\rm trace}}
\newcommand{\unique}{{\rm unique}}
\newcommand{\arb}{{\rm ARB}}
\newcommand{\rmse}{{\rm RMSE}}
\newcommand{\D}{\mathcal{D}}

\newcommand{\Cov}[1]{\textrm{Cov}[#1]}

\newtheorem{thm}{Theorem}
\newtheorem{lemm}{Lemma}
\newtheorem{defi}{Definition}

\usepackage{microtype}
\usepackage{graphicx}
\usepackage{booktabs} 
\usepackage{subfig}

\usepackage{hyperref}

\usepackage{listings}
\usepackage{xcolor}
\newxyColor{lightred}{1.0 0.7 0.7}{rgb}{}

\providecommand{\keywords}[1]{\noindent\textbf{Keywords:} #1}

\begin{document}

\title{A Novel Bayesian Approach for Latent Variable Modeling
	from Mixed Data with Missing Values
}


\author{\textbf{Ruifei Cui, Ioan Gabriel Bucur, Perry Groot, Tom Heskes} \\
	Institute for Computing and Information Sciences \\
	Radboud University Nijmegen \\
	The Netherlands \\
	\{r.cui, g.bucur, perry.groot, t.heskes\}@science.ru.nl
}




\maketitle

\begin{abstract}
	We consider the problem of learning parameters of latent variable models from mixed (continuous and ordinal) data with missing values. We propose a novel Bayesian Gaussian copula factor (BGCF) approach that is consistent under certain conditions and that is quite robust to the violations of these conditions. In simulations, BGCF substantially outperforms two state-of-the-art alternative approaches. An illustration on the `Holzinger \& Swineford 1939' dataset indicates that BGCF is favorable over the so-called robust maximum likelihood (MLR) even if the data match the assumptions of MLR.
\end{abstract} \\

\keywords{latent variables; Gaussian copula factor model; parameter learning; mixed data; missing values}

\section{Introduction}

In psychology, social sciences, and many other fields, researchers are usually interested in ``latent'' variables that cannot be measured directly, e.g., depression, anxiety, or intelligence. To get a grip on these latent concepts, one commonly-used strategy is to construct a measurement model for such a latent variable, in the sense that domain experts design multiple ``items'' or ``questions'' that are considered to be indicators of the latent variable. For exploring evidence of construct validity in theory-based instrument construction, confirmatory factor analysis (CFA) has been widely studied~\citep{joreskog1969general,castro2015likelihood,li2016confirmatory}. In CFA, researchers start with several hypothesised latent variable models that are then fitted to the data individually, after which the one that fits the data best is picked to explain the observed phenomenon. In this process, the fundamental task is to learn the parameters of a hypothesised model from observed data, which is the focus of this paper. For convenience, we simply refer to these hypothesised latent variable models as CFA models from now on.

The most common method for parameter estimation in CFA models is maximum likelihood (ML), because of its attractive statistical properties (consistency, asymptotic normality, and efficiency). The ML method, however, relies on the assumption that observed variables follow a multivariate normal distribution~\citep{joreskog1969general}. When the normality assumption is not deemed empirically
tenable, ML may not only reduce the accuracy of parameter estimates, but may also yield misleading conclusions drawn from empirical data~\citep{li2016confirmatory}. To this end, a robust version of ML was introduced for CFA models when the normality assumption is slightly or moderately violated~\citep{kaplan2008structural}, but still requires the observations to be continuous. In the real world, the indicator data in questionnaires are usually measured on an ordinal scale (resulting in a bunch of ordered categorical variables, or simply ordinal variables)~\citep{poon2012latent}, in which neither normality nor continuity is plausible~\citep{lubke2004applying}. In such cases, diagonally weighted least squares (DWLS in LISREL; WLSMV or robust WLS in M\textit{plus}) has been suggested to be superior to the ML method and is usually considered to be preferable over other methods~\citep{barendse2015using,li2016confirmatory}. 

However, there are two major issues that the existing approaches do not consider. One is the mixture of continuous and ordinal data. As we mentioned above ordinal variables are omnipresent in questionnaires, whereas sensor data are usually continuous. Therefore, a more realistic case in real applications is mixed continuous and ordinal data. A second important issue concerns missing values. In practice, all branches of experimental science are plagued by missing values~\citep{rja1987statistical}, e.g., failure of sensors, or unwillingness to answer certain questions in a survey. A straightforward idea in this case is to combine missing values techniques with existing parameter estimation approaches, e.g., performing listwise-deletion or pairwise-deletion first on the original data and then applying DWLS to learn parameters of a CFA model. However, such deletion methods are only consistent when the data are \textit{missing completely at random} (MCAR), which is a rather strong assumption~\citep{rubin1976inference}, and cannot transfer the sampling variability incurred by missing values to follow-up studies. The two modern missing data techniques, maximum likelihood and multiple imputation, are valid under a less restrictive assumption, \textit{missing at random} (MAR)~\citep{schafer2002missing}, but they require the data to be multivariate normal. 

Therefore, there is a strong demand for an approach that is not only valid under MAR but also works for mixed continuous and ordinal data. For this purpose, we propose a novel Bayesian Gaussian copula factor (BGCF) approach, in which a Gibbs sampler is used to draw pseudo Gaussian data in a latent space restricted by the observed data (unrestricted if that value is missing) and draw posterior samples of parameters given the pseudo data, iteratively. We prove that this approach is consistent under MCAR and empirically show that it works quite well under MAR.


The rest of this paper is organized as follows. Section~\ref{sec:background} reviews background knowledge and related work. Section~\ref{sec:method} gives the definition of a Gaussian copula factor model and presents our novel inference procedure for this model. Section~\ref{sec:simulation} compares our BGCF approach with two alternative approaches on simulated data, and Section~\ref{sec:application} gives an illustration on the `Holzinger \& Swineford 1939' dataset. Section~\ref{sec:conclusion} concludes this paper and provides some discussion.

\section{Background} \label{sec:background}

This section reviews basic missingness mechanisms and related work on parameter estimation in CFA models.

\subsection{Missingness Mechanism}

Following~\citet{rubin1976inference}, let $\bm{Y} = (y_{ij}) \in \mathbb{R}^{n \times p}$ be a data matrix with the rows representing independent samples, and $ \bm{R} = (r_{ij}) \in \{0,1\}^{n \times p}$ be a matrix of indicators, where $r_{ij} = 1$ if $y_{ij}$ was observed and $r_{ij} = 0$ otherwise. $\bm{Y}$ consists of two parts, $\bm{Y}_{obs}$ and $\bm{Y}_{miss}$, representing observed and missing elements in $\bm{Y}$ respectively. When the missingness does not depend on the data, i.e., $P(\bm{R}|\bm{Y}, \theta) = P(\bm{R}|\theta)$ with $\theta$ denoting unknown parameters, the data are said to be \emph{missing completely at random} (MCAR), which is a special case of a more realistic assumption called \emph{missing at random} (MAR). MAR allows the dependency between missingness and observed values, i.e., $P(\bm{R}|\bm{Y}, \theta) = P(\bm{R}|\bm{Y}_{obs},\theta)$. For example, all people in a group are required to take a blood pressure test at time point 1, while only those whose values at time point 1 lie in the abnormal range need to take the test at time point 2. This results in some missing values at time point 2 that are MAR. 

\subsection{Parameter Estimation in CFA Models}

When the observations follow a multivariate normal distribution, maximum likelihood (ML) is the mostly-used method. It is equivalent to minimizing the discrepancy function $F_{\rm{ML}}$~\citep{joreskog1969general}:
\[
F_{\rm{ML}} = \ln\lvert\Sigma(\theta)\lvert + \tr[S\Sigma^{-1}(\theta)] - \ln\lvert S\lvert - p \:,
\]
where $\theta$ is the vector of model parameters, $\Sigma(\theta)$ is the model-implied covariance matrix, $S$ is the sample covariance matrix, and $p$ is the number of observed variables in the model. When the normality assumption is violated either slightly or moderately, robust ML (MLR) offers an alternative. Here parameter estimates are still obtained using the asymptotically unbiased ML estimator, but standard errors are statistically corrected to enhance the robustness of ML against departures from normality~\citep{kaplan2008structural,muthen2010mplus}. Another method for continuous nonnormal data is the so-called asymptotically distribution free method, which is a weighted least squares (WLS)
method using the inverse of the asymptotic covariance matrix of the sample variances and
covariances as a weight matrix~\citep{browne1984asymptotically}.

When the observed data are on ordinal scales, \citet{muthen1984general} proposed a three-stage approach. It assumes that a normal latent variable $x^*$ underlies an observed ordinal variable $x$, i.e.,
\[
x = m, \mbox{~if~} \tau_{m-1} < x^* < \tau_m \:,
\]
where $m$ $(=1,2,...,c)$ denotes the observed values of $x$, $\tau_m$ are thresholds $(-\infty=\tau_0 < \tau_1 < \tau_2 < ... < \tau_c = +\infty)$, and $c$ is the number of categories. The thresholds and polychoric correlations are estimated from the bivariate contingency table in the first two stages~\citep{olsson1979maximum,joreskog2005structural}. Parameter estimates and the associated standard errors are then obtained by minimizing the weighted least squares
fit function $F_{\rm{WLS}}$:
\[
F_{\rm{WLS}} = [s-\sigma(\theta)]^T\bm{W}^{-1}[s-\sigma(\theta)]\:,
\]
where $\theta$ is the vector of model parameters, $\sigma(\theta)$ is the model-implied vector containing the nonredundant vectorized elements of $\Sigma(\theta)$, $s$ is the vector
containing the estimated polychoric correlations, and the weight matrix $\bm{W}$ is the asymptotic covariance matrix of the
polychoric correlations. A mathematically simple form of the WLS estimator, the unweighted least squares (ULS), arises when the matrix $\bm{W}$ is replaced with the identity matrix $\bm{I}$. Another variant of WLS is the diagonally weighted least squares (DWLS), in which only the diagonal elements of $\bm{W}$ are used in the fit function~\citep{muthen1997robust,muthen2010mplus}, i.e.,
\[
F_{\rm{DWLS}} = [s-\sigma(\theta)]^T\bm{W}^{-1}_{\rm{D}}[s-\sigma(\theta)]\:,
\]
where $\bm{W}^{-1}_{\rm{D}} = \diag(\bm{W})$ is the diagonal weight matrix. Various recent simulation studies have shown that DWLS is favorable compared to WLS, ULS, as well as the ML-based methods for ordinal data~\citep{barendse2015using,li2016confirmatory}.


\section{Method}
\label{sec:method}

In this section, we introduce the Gaussian copula factor model and propose a Bayesian inference procedure for this model. Then, we theoretically analyze the identifiability and prove the consistency of our procedure.

\subsection{Gaussian Copula Factor Model}\label{sec:model}

\begin{defi}[Gaussian Copula Factor Model] \label{def:GCFM}
	Consider a latent random (factor) vector $\bm{\eta} = (\eta_1,\ldots,\eta_k)^T$, a response random vector $\bm{Z} = (Z_1,\ldots,Z_p)^T$ and an observed random vector $\bm{Y}=(Y_1,\ldots,Y_p)^T$, satisfying
	\begin{gather}
		\label{eq:GCFM_latent}
		\bm{\eta} \sim \mathcal{N}(0,C),  \\
		\label{eq:GCFM_response}
		\bm{Z} = \Lambda \bm{\eta} + \bm{\epsilon}, \\
		\label{eq:GCFM_observe}
		Y_j = F_j^{-1}\big(\Phi\big[Z_j/\sigma(Z_j)\big]\big), \forall j = 1,\ldots,p,
	\end{gather}
	with $C$ a correlation matrix over factors, $\Lambda = (\lambda_{ij})$ a $p \times k$ matrix of factor loadings ($k \leq p$), $\bm{\epsilon} \sim \mathcal{N}(0,D)$ residuals with $D = \diag (\sigma_1^2,\ldots,\sigma_p^2)$, $\sigma(Z_j)$ the standard deviation of $Z_j$, $\Phi(\cdot)$ the cumulative distribution function (CDF) of the standard Gaussian, and ${F_{j}}^{-1}(t) = \inf\{ x: F_{j}(x) \geq t\}$ the pseudo-inverse of a CDF $F_j(\cdot)$. Then this model is called a \emph{Gaussian copula factor model}.
\end{defi}

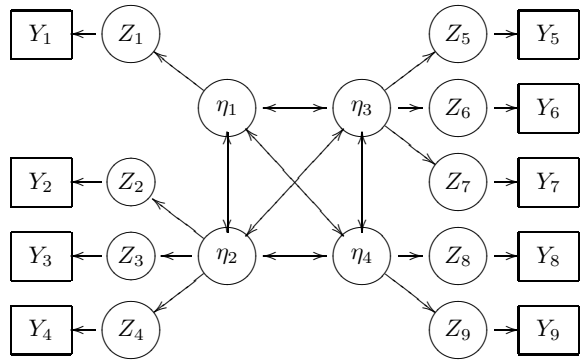
\begin{figure}[h]
	\centering
	\begin{tabular}{c}
		\parbox[t]{0.5\textwidth}{%
			\scalebox{0.9}{\centerline{\xymatrix @R=1.1em @C=1em{
						*=<2.5em,2em>[F]{Y_1} & *=<3em,2em>[o][F]{Z_1} \ar[l] & & & & *=<3em,2em>[o][F]{Z_{5}} \ar[r] &
						*=<2.5em,2em>[F]{Y_{5}} \\
						& &
						*=<3em,2em>[o][F]{\eta_1} \ar[ul] \ar@{<->}[rrdd] \ar@{<->}[dd] \ar@{<->}[rr]& &
						*=<3em,2em>[o][F]{\eta_3} \ar[ur] \ar[r] \ar[rd] \ar@{<->}[dd] \ar@{<->}[lldd]&
						*=<3em,2em>[o][F]{Z_{6}} \ar[r] &
						*=<2.5em,2em>[F]{Y_{6}} \\
						*=<2.5em,2em>[F]{Y_2}& *=<2.5em,2em>[o][F]{Z_2} \ar[l]
						& & & & *=<3em,2em>[o][F]{Z_{7}} \ar[r] & *=<2.5em,2em>[F]{Y_{7}} \\
						*=<2.5em,2em>[F]{Y_3}& *=<2.5em,2em>[o][F]{Z_3} \ar[l] &
						*=<3em,2em>[o][F]{\eta_2} \ar[ld] \ar[lu] \ar[l] \ar@{<->}[rr] & &
						*=<3em,2em>[o][F]{\eta_4} \ar[rd]  \ar[r]&
						*=<3em,2em>[o][F]{Z_{8}} \ar[r] &
						*=<2.5em,2em>[F]{Y_{8}}
						\\
						*=<2.5em,2em>[F]{Y_4} &
						*=<3em,2em>[o][F]{Z_{4}} \ar[l] & & & &
						*=<3em,2em>[o][F]{Z_{9}} \ar[r] &
						*=<2.5em,2em>[F]{Y_{9}} \\
		}}}}
	\end{tabular}
	\caption{Gaussian copula factor model.}
	\label{GaussianCopulaModelDemo}
\end{figure}

The model is also defined in~\citet{murray2013bayesian}, but the authors restrict the factors to be independent of each other while we allow for their interactions. Our model is a combination of a Gaussian factor model (from $\bm{\eta}$ to $\bm{Z}$) and a Gaussian copula model (from $\bm{Z}$ to $\bm{Y}$). The first part allows us to model the latent concepts that are measured by multiple indicators, and the second part provides a good way to model diverse types of variables (depending on $F_j(\cdot)$ in Equation~\ref{eq:GCFM_observe}, $Y_j$ can be either continuous or ordinal). Figure~\ref{GaussianCopulaModelDemo} shows an example of the model. Note that we allow the special case of a factor having a single indicator, e.g., $\eta_1 \rightarrow Z_1 \rightarrow Y_1$, because this allows us to incorporate other (explicit) variables (such as age and income) into our model. In this special case, we set $\lambda_{11} = 1$ and $\epsilon_1 = 0$, thus $Y_1 = F_1^{-1}(\Phi[\eta_1])$.

In the typical design for questionnaires, one tries to get a grip on a latent concept through a particular set of well-designed questions~\citep{martinez2006procedure,byrne2013structural}, which implies that a factor (latent concept) in our model is connected to multiple indicators (questions) while an indicator is only used to measure a single factor, as shown in Figure~\ref{GaussianCopulaModelDemo}. This kind of measurement model is called a \emph{pure measurement model} (Definition 8 in~\citet{silva2006learning}). Throughout this paper, we assume that all measurement models are pure, which indicates that there is only a single non-zero entry in each row of the factor loadings matrix $\Lambda$. This inductive bias about the sparsity pattern of $\Lambda$ is fully motivated by the typical design of a measurement model.

In what follows, we transform the Gaussian copula factor model into an equivalent model that is used for inference in the next subsection. We consider an integrated $(p + k)$-dimensional random vector $\bm{X} = (\bm{Z}^T, \bm{\eta}^T)^T$, which is still multivariate Gaussian, and obtain its covariance matrix
\begin{equation}
	\label{eq:cov_X}
	\Sigma = \begin{bmatrix}
		\Lambda C \Lambda^T + D & \Lambda C \\
		C \Lambda^T & C \\
	\end{bmatrix} \:,
\end{equation}
and precision matrix
\begin{equation}
	\label{eq:integratedPrecison}
	\Omega = \Sigma^{-1} = \begin{bmatrix}
		D^{-1} & -D^{-1} \Lambda \\
		-\Lambda^T D^{-1} & C^{-1} + \Lambda^T D^{-1} \Lambda \\
	\end{bmatrix} \:.
\end{equation}

Since $D$ is diagonal and $\Lambda$ only has one non-zero entry per row, $\Omega$ contains many intrinsic zeros. The sparsity pattern of such $\Omega = (\omega_{ij})$ can be represented by an undirected graph $G = (\bm{V}, \bm{E})$, where $(i,j) \not\in \bm{E}$ whenever $\omega_{ij} = 0$ by construction. 
Then, a Gaussian copula factor model can be transformed into an equivalent model controlled by a single precision matrix $\Omega$, which in turn is constrained by $G$, i.e., $P(\bm{X}|C,\Lambda,D) = P(\bm{X}|\Omega_G)$.

\begin{defi}[$G$-Wishart Distribution]
	Given an undirected graph $G = (\bm{V},\bm{E})$, a zero-constrained random matrix $\Omega$ has a $G$-Wishart distribution, if its density function is
	$$
	p(\Omega|G) = \frac{|\Omega|^{(\nu - 2)/2}}{I_G(\nu, \Psi)} \exp \bigg[-\frac{1}{2} \tr(\Psi \Omega)\bigg] \mathbbm{1}_{\Omega \in M^+(G)},
	$$
	with $M^+(G)$ the space of symmetric positive definite matrices with off-diagonal elements $\omega_{ij} = 0$ whenever $(i,j) \not\in \bm{E}$, $\nu$ the number of degrees of freedom, $\Psi$ a scale matrix, $I_G(\nu, \Psi)$ the normalizing constant, and $\mathbbm{1}$ the indicator function~\citep{roverato2002hyper}.
\end{defi}

The $G$-Wishart distribution is the conjugate prior of precision matrices $\Omega$ that are constrained by a graph $G$~\citep{roverato2002hyper}. That is, given the $G$-Wishart prior, i.e., $P(\Omega|G) = \Wish_G(\nu_0, \Psi_0)$ and data $\bm{X} = (\bm{x_1},\ldots,\bm{x_n})^T$ drawn from $\mathcal{N}(0,\Omega^{-1})$, the posterior for $\Omega$ is another $G$-Wishart distribution:
\begin{equation*}
	\label{posteriorDistribution}
	P(\Omega | G, \bm{X}) = \Wish_G (\nu_0 + n, \Psi_0 + \bm{X}^T \bm{X}).
\end{equation*}
When the graph $G$ is fully connected, the $G$-Wishart distribution reduces to a Wishart distribution~\citep{murphy2007conjugate}. Placing a $G$-Wishart prior on $\Omega$ is equivalent to placing an inverse-Wishart on $C$, a product of multivariate normals on $\Lambda$, and an inverse-gamma on the diagonal elements of $D$. With a diagonal scale matrix $\Psi_0$ and the number of degrees of freedom $\nu_0$ equal to the number of factors plus one, the implied marginal densities between any pair of factors are uniformly distributed between $[-1,1]$~\citep{barnard2000modeling}.

\subsection{Inference for Gaussian Copula Factor Model}

We first introduce the inference procedure for complete mixed data and incomplete Gaussian data respectively, based on which the procedure for mixed data with missing values is then derived. From this point on, we use $S$ to denote the correlation matrix over the response vector $\bm{Z}$.

\subsubsection{Mixed Data without Missing Values} \label{sec:infer_mixed_complete}

For a Gaussian copula model,~\citet{hoff2007extending} proposed a likelihood that only concerns the ranks among observations, which is derived as follows. Since the transformation $Y_j = F_j^{-1}\big(\Phi\big[Z_j\big]\big)$ is non-decreasing, observing $\bm{y}_j = (y_{1,j},\ldots,y_{n,j})^T$ implies a partial ordering on $\bm{z}_j = (z_{1,j},\ldots,z_{n,j})^T$, i.e., $\bm{z}_j$ lies in the space restricted by $\bm{y}_j$:
\[
\D(\bm{y}_j) = \{\bm{z}_j \in \mathbb{R}^n: y_{i,j} < y_{k,j} \Rightarrow z_{i,j} < z_{k,j}\} \:.
\]
Therefore, observing $\bm{Y}$ suggests that $\bm{Z}$ must be in
\[
\D(\bm{Y}) = \{\bm{Z} \in \mathbb{R}^{n \times p}: \bm{z}_j \in \D(\bm{y}_j), \forall j = 1,\ldots,p\} \:.
\]
Taking the occurrence of this event as the data, one can compute the following likelihood~\cite{hoff2007extending}
\begin{align*}
	P(\bm{Z} \in \D(\bm{Y})|S,F_1,\ldots,F_p) = P(\bm{Z} \in \D(\bm{Y})|S).
\end{align*}

Following the same argumentation, the likelihood in our Gaussian copula factor model reads
\begin{equation*}
	P(\bm{Z} \in \D(\bm{Y})|\bm{\eta},\Omega,F_1,\ldots,F_p) = P(\bm{Z} \in \D(\bm{Y})|\bm{\eta},\Omega), \:
\end{equation*}
which is independent of the margins $F_j$.

For the Gaussian copula factor model, inference for the precision matrix $\Omega$ of the vector $\bm{X} = (\bm{Z}^T, \bm{\eta}^T)^T$ can now proceed via construction of a Markov chain having its stationary distribution equal to $P(\bm{Z},\bm{\eta},\Omega|\bm{Z} \in \D(\bm{Y}),G)$, where we ignore the values for $\bm{\eta}$ and $\bm{Z}$ in our samples. The prior graph $G$ is uniquely determined by the sparsity pattern of the loading matrix $\Lambda = (\lambda_{ij})$ and the residual matrix $D$ (see Equation \ref{eq:integratedPrecison}), which in turn is uniquely decided by the pure measurement models. The Markov chain can be constructed by iterating the following three steps:
\begin{enumerate}
	\item \textbf{Sample $\bm{Z}$}: $\bm{Z} \sim P(\bm{Z}|\bm{\eta},\bm{Z} \in \D(\bm{Y}),\Omega)$; \\
	Since each coordinate $Z_j$ directly depends on only one factor, i.e., $\eta_q$ such that $\lambda_{jq} \neq 0$, we can sample each of them independently through
	$ Z_j \sim P(Z_j|\eta_q,\bm{z}_j \in \D(\bm{y}_j),\Omega) $.
	\item \textbf{Sample $\bm{\eta}$}: $\bm{\eta} \sim P(\bm{\eta}|\bm{Z},\Omega)$;
	\item \textbf{Sample $\Omega$}: $\Omega \sim P(\Omega|\bm{Z},\bm{\eta},G)$.
\end{enumerate}

\subsubsection{Gaussian Data with Missing Values} \label{sec:infer_Gaussian}

Suppose that we have Gaussian data $\bm{Z}$ consisting of two parts, $\bm{Z}_{obs}$ and $\bm{Z}_{miss}$, denoting observed and missing values in $\bm{Z}$ respectively. The inference for the correlation matrix of $\bm{Z}$ in this case can be done via the so-called data augmentation technique that is also a Markov chain Monte Carlo procedure and has been proven to be consistent under MAR~\citep{schafer1997analysis}. This approach iterates the following two steps to impute missing values (Step 1) and draw correlation matrix samples from the posterior (Step 2):
\begin{enumerate}
	\item  $\bm{Z}_{miss} \sim P(\bm{Z}_{miss}|\bm{Z}_{obs},S)$ ;
	\item  $S \sim P(S|\bm{Z}_{obs},\bm{Z}_{miss})$.
\end{enumerate}	

\subsubsection{Mixed Data with Missing Values}

For the most general case of mixed data with missing values, we combine the procedures of Sections~\ref{sec:infer_mixed_complete} and~\ref{sec:infer_Gaussian} into the following four-step inference procedure:
\begin{enumerate}
	\item  $\bm{Z}_{obs} \sim P(\bm{Z}_{obs}|\bm{\eta},\bm{Z}_{obs} \in \D(\bm{Y}_{obs}),\Omega)$;
	\item  $\bm{Z}_{miss} \sim P(\bm{Z}_{miss}|\bm{\eta},\bm{Z}_{obs},\Omega)$;
	\item  $\bm{\eta} \sim P(\bm{\eta}|\bm{Z}_{obs},\bm{Z}_{miss},\Omega)$;
	\item  $\Omega \sim P(\Omega|\bm{Z}_{obs},\bm{Z}_{miss},\bm{\eta},G)$.
\end{enumerate}

A Gibbs sampler that achieves this Markov chain is summarized in Algorithm~\ref{GS_GCFM} and implemented in \textsf{R}.\footnote{The code including those used in simulations and real-world applications is provided in \url{https://github.com/cuiruifei/CopulaFactorModel}.} Note that we put Step 1 and Step 2 together in the actual implementation since they share some common computations (lines 2 - 4). The difference between the two steps is that the values in Step 1 are drawn from a space restricted by the observed data (lines 5 - 13) while the values in Step 2 are drawn from an unrestricted space (lines 14 - 17). Another important point is that we need to relocate the data such that the mean of each coordinate of $\bm{Z}$ is zero (line 20). This is necessary for the algorithm to be sound because the mean may shift when missing values depend on the observed data (MAR).

\begin{algorithm}[!t]
	\caption{Gibbs sampler for Gaussian copula factor model with missing values}
	\label{GS_GCFM}
	\begin{algorithmic}[1]
		\REQUIRE Prior graph $G$, observed data $\bm{Y}$. \\
		\# \textbf{Step 1} and \textbf{Step 2}:
		\FOR{$j \in \{1,\ldots,p\}$}
		\STATE $q=$ factor index of $Z_j$
		\STATE $ a = \Sigma_{[j,q+p]} / \Sigma_{[q+p,q+p]}$
		\STATE $\sigma_j^2 = \Sigma_{[j,j]}-a \times \Sigma_{[q+p,j]}$ \\
		\# \textbf{Step 1}: $\bm{Z}_{obs} \sim P(\bm{Z}_{obs}|\bm{\eta},\bm{Z}_{obs} \in \D(\bm{Y}_{obs}),\Omega)$
		\FOR{$y \in \unique \{y_{1,j},\ldots,y_{n,j}\}$}
		\STATE $z_l = \max\{z_{i,j}:y_{i,j}<y\}$
		\STATE $z_u = \min\{z_{i,j}:y<y_{i,j}\}$
		\FOR{$i$ such that $\: y_{i,j} = y$}
		\STATE $\mu_{i,j} = \bm{\eta}_{[i,q]} \times a$
		\STATE $u_{i,j} \sim \mathcal{U}\big(\Phi\big[\frac{z_l-\mu_{i,j}}{\sigma_j}\big],\Phi\big[\frac{z_u-\mu_{i,j}}{\sigma_j}\big]\big)$
		\STATE $z_{i,j} = \mu_{i,j} + \sigma_j \times \Phi^{-1}(u_{i,j})$
		\ENDFOR
		\ENDFOR \\
		\# \textbf{Step 2}: $\bm{Z}_{miss} \sim P(\bm{Z}_{miss}|\bm{\eta},\bm{Z}_{obs},\Omega)$
		\FOR{$i$ such that $y_{i,j} \in \bm{Y}_{miss}$}
		\STATE $\mu_{i,j} = \bm{\eta}_{[i,q]} \times a$
		\STATE $z_{i,j} \sim \mathcal{N}(\mu_{i,j}, \sigma_j^2)$
		\ENDFOR	
		\ENDFOR
		\STATE $\bm{Z} = (\bm{Z}_{obs}, \bm{Z}_{miss})$ 
		\STATE $\bm{Z} = (\bm{Z}^T - \bm{\mu})^T$, with $\bm{\mu}$ the mean vector of $\bm{Z}$ \\
		\# \textbf{Step 3}: $\bm{\eta} \sim P(\bm{\eta}|\bm{Z},\Omega)$
		\STATE $A = \Sigma_{[\bm{\eta},\bm{Z}]}\Sigma_{[\bm{Z},\bm{Z}]}^{-1}$
		\STATE $B = \Sigma_{[\bm{\eta},\bm{\eta}]}-A\Sigma_{[\bm{Z},\bm{\eta}]}$
		\FOR{$i \in \{1,\ldots,n\}$}		
		\STATE $\bm{\mu}_i = (\bm{Z}_{[i,:]}A^T)^T$
		\STATE  $\bm{\eta}_{[i,:]} \sim \mathcal{N}(\bm{\mu}_i, B)$
		\ENDFOR
		\STATE $\bm{\eta}_{[:,j]} = \bm{\eta}_{[:,j]} \times \sign(\Cov{\bm{\eta}_{[:,j]}, \bm{Z}_{[:,f(j)]}}), \: \forall j$, where $f(j)$ is the index of the first indicator of $\eta_j$.  \\
		\# \textbf{Step 4}: $\Omega \sim P(\Omega|\bm{Z},\bm{\eta},G)$
		\STATE $\bm{X} = (\bm{Z}, \bm{\eta})$
		\STATE  $\Omega \sim \Wish_G(\nu_0 + n, \Psi_0 + \bm{X}^T\bm{X})$
		\STATE $\Sigma = \Omega^{-1}$
		\STATE $\Sigma_{ij} = \Sigma_{ij}/\sqrt{\Sigma_{ii}\Sigma_{jj}}, \forall i,j$
	\end{algorithmic}
\end{algorithm}

By iterating the steps in Algorithm~\ref{GS_GCFM}, we can draw correlation matrix samples over the integrated random vector $\bm{X}$, denoted by $\{\Sigma^{(1)},\ldots, \Sigma^{(m)}\}$. The mean over all the samples is a natural estimate of the true $\Sigma$, i.e.,
\begin{equation}
	\hat{\Sigma} = \dfrac{1}{m}\sum_{i = 1}^{m} \Sigma^{(i)} \:.\label{eq:Sigma}
\end{equation}
Based on Equations (\ref{eq:cov_X}) and (\ref{eq:Sigma}), we obtain estimates of the parameters of interests:
\begin{align}
	&\hat{C} = \hat{\Sigma}_{[\bm{\eta}, \bm{\eta}]}; \nonumber \\ 
	&\hat{\Lambda} = \hat{\Sigma}_{[\bm{Z}, \bm{\eta}]} \hat{C}^{-1}\: ; \\
	&\hat{D} = \hat{S} - \hat{\Lambda}\hat{C}\hat{\Lambda}^T, \mbox{~~with~~} \hat{S} = \hat{\Sigma}_{[\bm{Z}, \bm{Z}]} \nonumber \:.
\end{align}
We refer to this procedure as a \emph{Bayesian Gaussian copula factor approach} (BGCF).

\subsection{Theoretical Analysis}

\paragraph{Identifiability of $C$} Without additional constraints, $C$ is non-identifiable~\citep{anderson1956statistical}. More precisely, given a decomposable matrix $S = \Lambda C \Lambda^T + D$, we can always replace $\Lambda$ with $\Lambda U$ and $C$ with $U^{-1} C U^{-T}$ to obtain an equivalent decomposition $S = (\Lambda U)(U^{-1} C U^{-T})(U^T \Lambda ^T) + D$, where $U$ is a $k \times k$ invertible matrix. Since $\Lambda$ only has one non-zero entry per row in our model, $U$ can only be diagonal to ensure that $\Lambda U$ has the same sparsity pattern as $\Lambda$ (see Lemma~\ref{lemm:lambda} in Appendix). Thus, from the same $S$, we get a class of solutions for $C$, i.e., $U^{-1} C U^{-1}$,  where $U$ can be any invertible diagonal matrix. In order to get a unique solution for $C$, we impose two sufficient identifying conditions: 1) restrict $C$ to be a correlation matrix; 2) force the first non-zero entry in each column of $\Lambda$ to be positive. See Lemma~\ref{lemm:identifiability} in Appendix for proof. Condition 1 is implemented via line 31 in Algorithm~\ref{GS_GCFM}. As for the second condition, we force the covariance between a factor and its first indicator to be positive (line 27), which is equivalent to Condition 2. Note that these conditions are not unique; one could choose one's favorite conditions to identify $C$, e.g., setting the first loading to 1 for each factor. The reason for our choice of conditions is to keep it consistent with our model definition where $C$ is a correlation matrix.

\paragraph{Identifiability of $\Lambda$ and $D$} Under the two conditions for identifying $C$, factor loadings $\Lambda$ and residual variances $D$ are also identified except for the case in which there exists one factor that is independent of all the others and this factor only has two indicators. For such a factor, we have 4 free parameters (2 loadings, 2 residuals) while we only have 3 available equations (2 variances, 1 covariance), which yields an underdetermined system. See Lemmas~\ref{lemm:identifiability_lambda} and~\ref{lemm:identifiability_D} in Appendix for detailed analysis. Once this happens, one could put additional constraints to guarantee a unique solution, e.g., by setting the variance of the first residual to zero. However, we would recommend to leave such an independent factor out (especially in association analysis) or study it separately from the other factors. 



Under sufficient conditions for identifying $C$, $\Lambda$, and $D$, our BGCF approach is consistent even with MCAR missing values. This is shown in Theorem~\ref{thm:consistency}, whose proof is provided in Appendix.

\begin{thm}[Consistency of the BGCF Approach]
	\label{thm:consistency} 	
	%
	Let $\bm{Y}_n=(\bm{y}_1,\ldots,\bm{y}_n)^T$ be independent observations drawn from a Gaussian copula factor model. If $\bm{Y}_n$ is complete (no missing data) or contains missing values that are missing completely at random, then
	\begin{gather*}
		\lim\limits_{n \to\infty} P\big(\hat{C}_n = C_0\big) = 1 \:, \\
		\lim\limits_{n \to\infty} P\big(\hat{\Lambda}_n = \Lambda_0\big) = 1 \:, \\
		\lim\limits_{n \to\infty} P\big(\hat{D}_n = D_0\big) = 1 \:,
	\end{gather*}
	where $\hat{C}_n$, $\hat{\Lambda}_n$, and $\hat{D}_n$ are parameters learned by BGCF, while $C_0$, $\Lambda_0$, and $D_0$ are the true ones.
\end{thm}

\section{Simulation Study} \label{sec:simulation}

In this section, we compare our BGCF approach with alternative approaches via simulations. 

\subsection{Setup}

\paragraph{Model specification} Following typical simulation studies on CFA models in the literature~\citep{yang2010confirmatory,li2016confirmatory}, we consider a correlated 4-factor model in our study. Each factor is measured by 4 indicators, since~\citet{marsh1998more} concluded that the accuracy of parameter estimates appeared to be optimal when the number of indicators per factor was four and marginally improved as the number increased. The interfactor correlations (off-diagonal elements of the correlation matrix $C$ over factors) are randomly drawn from [0.2, 0.4], which is considered a reasonable and empirical range in the applied literature~\citep{li2016confirmatory}. For the ease of reproducibility, we construct our $C$ as follows.
\begin{minipage}{\linewidth}
\begin{lstlisting}
set.seed(12345)
C <- matrix(runif(4^2, 0.2, 0.4), ncol=4)
C <- (C*lower.tri(C)) + t(C*lower.tri(C))
diag(C) <- 1
\end{lstlisting}
\end{minipage}
In the majority of empirical research and simulation studies~\citep{distefano2002impact}, reported standardized factor loadings range from 0.4 to 0.9. For facilitating interpretability and again reproducibility, each factor loading is set to 0.7. Each corresponding residual variance is then automatically set to 0.51 under a standardized solution in the population model, as done in~\citep{li2016confirmatory}.

\paragraph{Data generation} Given the specified model, one can generate data in the response space (the $\bm{Z}$ in Definition~\ref{def:GCFM}) via Equations \eqref{eq:GCFM_latent} and \eqref{eq:GCFM_response}. When the observed data (the $\bm{Y}$ in Definition~\ref{def:GCFM}) are ordinal, we discretize the corresponding margins into the desired number of categories. When the observed data are nonparanormal, we set the $F_j(\cdot)$ in Equation~\eqref{eq:GCFM_observe} to the CDF of a $\chi^2$-distribution with degrees of freedom \textit{df}. The reason for choosing a $\chi^2$-distribution is that we can easily use \textit{df} to control the extent of non-normality: a higher \textit{df} implies a distribution closer to a Gaussian. To fill in a certain percentage $\beta$ of missing values (we only consider MAR), we follow the procedure in~\citet{kolar2012estimating}, i.e., 
for $j = 1,\ldots,\lfloor p/2 \rfloor$, $i = 1,\ldots,n$: $y_{i,2*j}$ is missing if $z_{i,2*j-1} < \Phi^{-1}(2*\beta)$.
\paragraph{Evaluation metrics} We use average relative bias (ARB) and root mean squared error (RMSE) to examine the parameter estimates, which are defined as
\begin{gather*}
\arb = \dfrac{1}{r}\sum_{i = 1}^{r} \dfrac{\hat{\theta_i} - \theta_i}{\theta_i}\:, \:\:
\rmse = \sqrt{\dfrac{1}{r}\sum_{i = 1}^{r} (\hat{\theta_i} - \theta_i)^2} \:,
\end{gather*}
where $\hat{\theta_i}$ and $\theta_i$ represent the estimated and true values respectively. An ARB value less than 5\% is
interpreted as a \textit{trivial} bias, between 5\% and 10\% as a \textit{moderate} bias, and greater than 10\% as a \textit{substantial} bias~\citep{curran1996robustness}. Note that ARB describes an overall picture of average bias, that is, summing up bias
in a positive and a negative direction together. A smaller absolute value of ARB indicates better performance on average.

\subsection{Ordinal Data without Missing Values}

In this subsection, we consider ordinal complete data since this matches the assumptions of the diagonally weighted least squares (DWLS) method, in which we set the number of ordinal categories to be 4. We also incorporate the robust maximum likelihood (MLR) as an alternative approach, which was shown to be empirically tenable when the number of categories is more than 5~\citep{rhemtulla2012can,li2016confirmatory}. See Section~\ref{sec:background} for details of the two approaches.

Before conducting comparisons, we first check the convergence property of the Gibbs sampler used in our BGCF approach. Figure~\ref{fig:convergence} shows the RMSE of estimated interfactor correlations (left panel) and factor loadings (right panel) over 100 iterations for a randomly-drawn sample with sample size $n=500$. We see quite a good convergence of the Gibbs sampler, in which the burn-in period is only around 10. More experiments done for different numbers of categories and different random samples show that the burn-in is less than 20 on the whole across various conditions.
 
\begin{figure}[h]
	\centering	
	\includegraphics[scale=0.54]{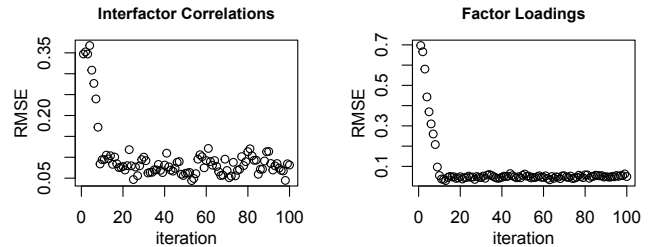}
	\caption{Convergence property of our Gibbs sampler over 100 iterations. Left panel: RMSE of interfactor correlations; Right panel: RMSE of factor loadings.}
	\label{fig:convergence}
\end{figure}

Now we evaluate the three involved approaches. Figure~\ref{fig:complete} shows the performance of BGCF, DWLS, and MLR over different sample sizes $n \in \{100, 200, 500, 1000\}$, providing the mean of ARB (left panel) and the mean of RMSE with 95\% confidence interval (right panel) over 100 experiments. From Figure~\ref{fig:corr_complete}, interfactor correlations are, on average, trivially biased (within two dashed lines) for all the three methods that in turn give indistinguishable RMSE regardless of sample sizes. From Figure~\ref{fig:loadings_complete}, MLR moderately underestimates the factor loadings, and performs worse than DWLS w.r.t. RMSE especially for a larger sample size, which confirms the conclusion in previous studies~\citep{barendse2015using,li2016confirmatory}. Most importantly, our BGCF approach outperforms DWLS in learning factor loadings especially for small sample sizes, even if the experimental conditions entirely match the assumptions of DWLS.

\begin{figure}[t]
	\centering	
	\subfloat[Interfactor Correlations]{\includegraphics[scale=0.7]{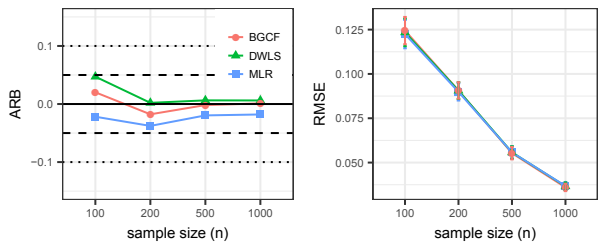}\label{fig:corr_complete}}
	\hfill
	\subfloat[Factor Loadings]{\includegraphics[scale=0.7]{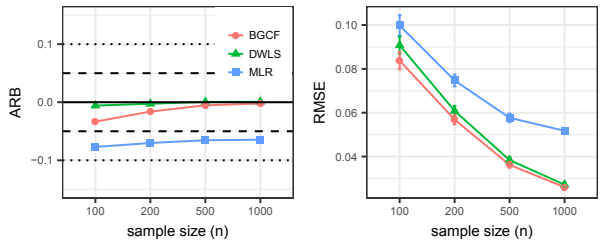}\label{fig:loadings_complete}}
	\caption{Results obtained by the Bayesian Gaussian copula factor (BGCF) approach, the diagonally weighted least squares (DWLS), and the robust maximum likelihood (MLR) on complete ordinal data (4 categories) over different sample sizes, showing the mean of ARB (left panel) and the mean of RMSE with 95\% confidence interval (right panel) over 100 experiments for (a) interfactor correlations and (b) factor loadings, where dashed lines and dotted lines in left panels denote $\pm 5\%$ and $\pm 10\%$ bias respectively.}
	\label{fig:complete}
\end{figure}

\subsection{Mixed Data with Missing Values}

In this subsection, we consider mixed nonparanormal and ordinal data with missing values, since some latent variables in real-world applications are measured by sensors that usually produce continuous but not necessarily Gaussian data. The 8 indicators of the first 2 factors (4 per factor) are transformed into a $\chi^2$-distribution with $df = 8$, which yields a slightly-nonnormal distribution (skewness is 1, excess kurtosis is 1.5)~\citep{li2016confirmatory}. The 8 indicators of the last 2 factors are discretized into ordinal with 4 categories. 

One alternative approach in such cases is DWLS with pairwise-deletion (PD), in which heterogeneous correlations (Pearson correlations between numeric variables, polyserial correlations between numeric and ordinal variables, and polychoric correlations between ordinal variables) are first computed based on pairwise complete observations, and then DWLS is used to estimate model parameters. A second alternative concerns the full information maximum likelihood (FIML)~\citep{arbuckle1996full,rosseel2012lavaan}, which first applies an EM algorithm to impute missing values and then uses MLR to learn model parameters.

Figure~\ref{fig:incomplete} shows the performance of BGCF, DWLS with PD, and FIML for $n = 500$ over different percentages of missing values $\beta \in \{0\%, 10\%, 20\%, 30\%\}$. First, despite a good performance with complete data ($\beta = 0\%$) DWLS (with PD) deteriorates significantly with an increasing percent of missing values especially for factor loadings, while BGCF and FIML show quite good scalability. Second, our BGCF approach overall outperforms FIML: indistinguishable for interfactor correlations but better for factor loadings.

\begin{figure}[t]
	\centering	
	\subfloat[Interfactor Correlations]{\includegraphics[scale=0.7]{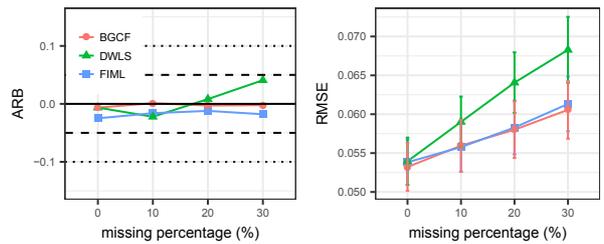}}\label{fig:corr_incomplete}
	\subfloat[Factor Loadings]{\includegraphics[scale=0.7]{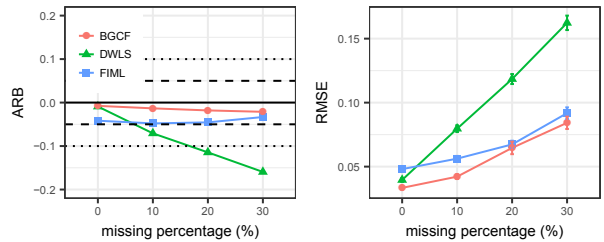}}\label{fig:loadings_incomplete}
	\caption{Results for $n = 500$ obtained by BGCF, DWLS with pairwise-deletion, and the full information maximum likelihood (FIML) on mixed nonparanormal (\textit{df} = 8) and ordinal (4 categories) data with different percentages of missing values, for the same experiments as in Figure~\ref{fig:complete}.}
	\label{fig:incomplete}
\end{figure}

Two more experiments are provided in Appendix. One concerns incomplete ordinal data with different numbers of categories, showing that BGCF is substantially favorable over DWLS (with PD) and FIML for learning factor loadings, which becomes more prominent with a smaller number of categories. Another one considers incomplete nonparanormal data with different extents of deviation from a Gaussian, which indicates that FIML is rather sensitive to the deviation and only performs well for a slightly-nonnormal distribution while the deviation has no influence on BGCF at all. See Appendix for more details.

\section{Application to Real-world Data} \label{sec:application}

In this section, we illustrate our approach on the `Holzinger \& Swineford 1939' dataset~\citep{holzinger1939study}, a classic dataset widely used in the literature and publicly available in the \textsf{R} package \textbf{lavaan}~\citep{rosseel2012lavaan}. The data consists of mental ability test scores of 301 students, in which we focus on 9 out of the original 26 tests as done in~\citet{rosseel2012lavaan}. A latent variable model that is often proposed to explore these 9 variables is a correlated 3-factor model shown in Figure~\ref{fig:HS_path}, where we rename the observed variables to ``Y1, Y2, \ldots, Y9'' for simplicity in visualization and to keep it identical to our definition of observed variables (Definition~\ref{def:GCFM}). The interpretation of these variables is given in the following list.

\begin{itemize}
	\item Y1: Visual perception;
	\item Y2: Cubes;
	\item Y3: Lozenges;
	\item Y4: Paragraph comprehension;
	\item Y5: Sentence completion;
	\item Y6: Word meaning;
	\item Y7: Speeded addition;
	\item Y8: Speeded counting of dots;
	\item Y9: Speeded discrimination straight and curved capitals.
\end{itemize}

\begin{figure*}[t]
	\centering
	\begin{tabular}{c}
		\parbox[t]{0.5\textwidth}{%
			\scalebox{1.1}{\centerline{\xymatrix @R=2em @C=3em{
						*=<2.5em,2em>[F]{Y_1} \ar@{<->}@(ul,ur)^{0.42} & *=<2.5em,2em>[F]{Y_2} \ar@{<->}@(ul,ur)^{0.83} & *=<2.5em,2em>[F]{Y_3} \ar@{<->}@(ul,ur)^{0.68} & & *=<2.5em,2em>[F]{Y_4} \ar@{<->}@(ru,rd)^{0.29}\\	
						& *=<4em,2em>[o][F]{visual} \ar[lu]|-{0.76} \ar[u]|-{0.41} \ar[ru]|-{0.57} \ar[rr]^{0.44} \ar[dd]^{0.47} & & *=<4em,2em>[o][F]{textual} \ar[ru]|-{0.84} \ar[r]|-{0.87} \ar[rd]|-{0.84} \ar[ll] \ar[lldd]^{0.28} & *=<2.5em,2em>[F]{Y_5} \ar@{<->}@(ru,rd)^{0.25} \\
						*=<2.5em,2em>[F]{Y_7} \ar@{<->}@(lu,ld)_{0.67} & & & & *=<2.5em,2em>[F]{Y_6} \ar@{<->}@(ru,rd)^{0.30} \\
						*=<2.5em,2em>[F]{Y_8} \ar@{<->}@(lu,ld)_{0.48} & *=<4em,2em>[o][F]{speed} \ar[lu]|-{0.58} \ar[l]|-{0.72} \ar[ld]|-{0.66} \ar[uu] \ar[rruu] & & &  \\
						*=<2.5em,2em>[F]{Y_9} \ar@{<->}@(lu,ld)_{0.57} & & & & 
		}}}}
	\end{tabular}
	\caption{Path diagram for the Holzinger \& Swineford data, in which latent variables are in ovals while observed variables are in squares, bidirected edges between latent variables denote correlation coefficients (interfactor correlations), directed edges denote factor loadings, and self-referring arrows denote residual variance, respectively. The edge weights in the graph are the model parameters learned by our BGCF approach.}
	\label{fig:HS_path}
\end{figure*}
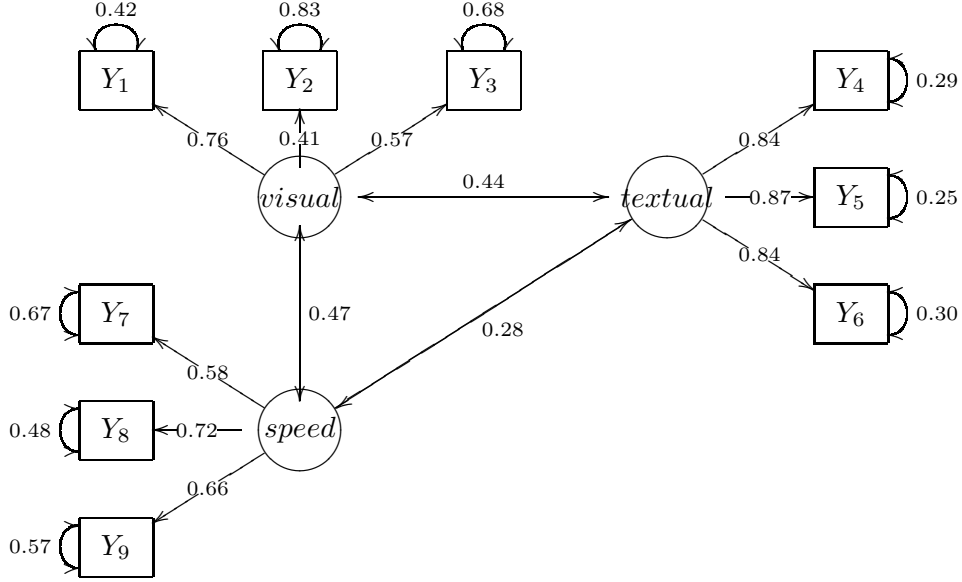

The summary of the 9 variables in this dataset is provided in Table~\ref{tab:summary_HS}, showing the number of unique values, skewness, and (excess) kurtosis for each variable. From the column of uniques values, we notice that the data are approximately continuous. The average of `absolute skewness' and `absolute excess kurtosis' over the 9 variables are around 0.40 and 0.54 respectively, which is considered to be slightly nonnormal~\citep{li2016confirmatory}. Therefore, we choose MLR as the alternative to be compared with our BGCF approach, since these conditions match the assumptions of MLR.

\begin{table}[h]
	\caption{The number of unique values, skewness, and (excess) kurtosis of each variable in the `HolzingerSwineford1939' dataset.}
	\label{tab:summary_HS}
	\begin{center}
		\begin{small}
			\begin{sc}
				\begin{tabular}{ccccc}
					\toprule
					Variables & Unique Values & Skewness & Kurtosis \\
					\midrule
					Y1 &  35 & -0.26 & 0.33 \\ 
					Y2 &  25 & 0.47 & 0.35 \\ 
					Y3 &  35 & 0.39 & -0.89 \\ 
					Y4 &  20 & 0.27 & 0.10 \\ 
					Y5 &  25 & -0.35 & -0.54 \\ 
					Y6 &  40 & 0.86 & 0.84 \\ 
					Y7 &  97 & 0.25 & -0.29 \\ 
					Y8 &  84 & 0.53 & 1.20 \\ 
					Y9 & 129 & 0.20 & 0.31 \\ 
					\bottomrule
				\end{tabular}
			\end{sc}
		\end{small}
	\end{center}
	\vskip -0.1in
\end{table}

We run our Bayesian Gaussian copula factor approach on this dataset. The learned parameter estimates are shown in Figure~\ref{fig:HS_path}, in which interfactor correlations are on the bidirected edges, factor loadings are in the directed edges, and unique variance for each variable is around the self-referring arrows. The parameters learned by the MLR approach are not shown here, since we do not know the ground truth so that it is hard to conduct a comparison between the two approaches.

In order to compare the BGCF approach with MLR quantitatively, we consider answering the question: ``What is the value of $Y_j$ when we observe the values of the other variables, denoted by $\bm{Y}_{\backslash j}$, given the population model structure in Figure~\ref{fig:HS_path}?". 

This is a regression problem but with additional constraints to obey the population model structure. The difference from a traditional regression problem is that we should learn the regression coefficients from the model-implied covariance matrix rather than the sample covariance matrix over observed variables.

\begin{itemize}
	\item For MLR, we first learn the model parameters on the training set, from which we extract the linear regression intercept and coefficients of $Y_j$ on $\bm{Y}_{\backslash j}$. Then we predict the value of $Y_j$ based on the values of $\bm{Y}_{\backslash j}$. See Algorithm~\ref{alg:MLR_reg} for pseudo code of this procedure.
	\item For BGCF, we first estimate the correlation matrix $\hat{S}$ over response variables (the $\bm{Z}$ in Definition~\ref{def:GCFM}) and the empirical CDF $\hat{F}_j$ of $Y_j$ on the training set. Then we draw latent Gaussian data $Z_j$ given $\hat{S}$ and $\bm{Y}_{\backslash j}$, i.e., $P(Z_j |\hat{S}, \bm{Z}_{\backslash j} \in \mathcal{D}(\bm{Y}_{\backslash j}))$. Lastly, we obtain the value of $Y_j$ from $Z_j$ via $\hat{F}_j$, i.e., $Y_j = \hat{F}_j^{-1} \big(\Phi[Z_j]\big)$. See Algorithm~\ref{alg:BGCF_reg} for pseudo code of this procedure. Note that we iterate the prediction stage (lines 7-8) for multiple times in the actual implementation to get multiple solutions to $Y_j^{(new)}$, then the average over these solutions is taken as the final predicted value of $Y_j^{(new)}$. This idea is quite similar to multiple imputation. 
\end{itemize}

\begin{algorithm}[H]
	\caption{Pseudo code of MLR for regression.}
	\label{alg:MLR_reg}
	\begin{algorithmic}[1]
		\STATE \textbf{Input:} $\bm{Y}^{(train)}$ and $\bm{Y}_{\backslash j}^{(new)}$.
		\STATE \textbf{Output:} $Y_j^{(new)}$.
		\STATE \textbf{Training Stage:}
		\STATE Fit the model using MLR on $\bm{Y}^{(train)}$;
		\STATE Extract the model-implied covariance matrix from the fitted model, denoted by $\hat{S}$;
		\STATE Extract regression coefficients $\bm{b}$ of $Y_j$ on $\bm{Y}_{\backslash j}$ from $\hat{S}$, that is, $\bm{b} = \hat{S}_{[\backslash j,\backslash j]}^{-1} \hat{S}_{[\backslash j,j]}$;
		\STATE Obtain the regression intercept $b_0$, that is, \\ $b_0 = \E (Y_j^{(train)}) - \bm{b} \cdot \E (\bm{Y}_{\backslash j}^{(train)})$.
		\STATE \textbf{Prediction Stage:}
		\STATE $Y_j^{(new)} = b_0 + \bm{b} \cdot \bm{Y}_{\backslash j}^{(new)}$.
	\end{algorithmic}
\end{algorithm}

\begin{algorithm}[H]
	\caption{Pseudo code of BGCF for regression.}
	\label{alg:BGCF_reg}
	\begin{algorithmic}[1]
		\STATE \textbf{Input:} $\bm{Y}^{(train)}$ and $\bm{Y}_{\backslash j}^{(new)}$.
		\STATE \textbf{Output:} $Y_j^{(new)}$.
		\STATE \textbf{Training Stage:}
		\STATE Apply BGCF to learn the correlation matrix over response variables, i.e., $\hat{S} = \hat{\Sigma}_{[\bm{Z}, \bm{Z}]}$;
		\STATE Learn the empirical cumulative distribution function of $Y_j$, denoted by $\hat{F}_j$.
		\STATE \textbf{Prediction Stage:}
		\STATE Sample $Z_j^{(new)}$ from $P(Z_j^{(new)} |\hat{S}, \bm{Z}_{\backslash j} \in \mathcal{D}(\bm{Y}_{\backslash j}))$;
		\STATE Obtain $Y_j^{(new)}$, i.e., $Y_j^{(new)} = \hat{F}_j^{-1} \big(\Phi[Z_j^{(new)}]\big)$.
	\end{algorithmic}
\end{algorithm}

\begin{figure*}
	\centering
	\includegraphics[scale = 0.9]{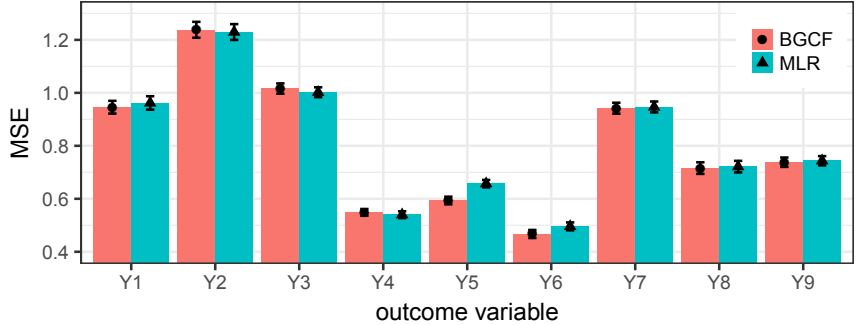}
	\caption{MSE obtained by BGCF and MLR when we take each $Y_j$ as outcome variable (the others as predictors) alternately, showing the mean over 100 experiments (10 times 10-fold cross validation) with error bars representing a standard error.}
	\label{fig:HS_MSE}
\end{figure*}

The mean squared error (MSE) is used to evaluate the prediction accuracy, where we repeat a 10-fold cross validation for 10 times (thus 100 MSE estimates totally). Also, we take $Y_j$ as the outcome variable alternately while treating the others as predictors (thus 9 tasks totally). Figure~\ref{fig:HS_MSE} provides the results of BGCF and MLR for all the 9 tasks, showing the mean of MSE with a standard error represented by error bars over the 100 estimates. We see that BGCF outperforms MLR for Tasks 5 and 6 although they perform indistinguishably for the other tasks. The advantage of BGCF over MLR is encouraging, considering that the experimental conditions match the assumptions of MLR. More experiments are done (not shown) after we make the data moderately or substantially nonnormal, suggesting that BGCF is significantly favorable to MLR, as expected.

\section{Summary and Discussion} \label{sec:conclusion}

In this paper, we proposed a novel Bayesian Gaussian copula factor (BGCF) approach for learning parameters of CFA models that can handle mixed continuous and ordinal data with missing values. We analyzed the separate identifiability of interfactor correlations $C$, factor loadings $\Lambda$, and residual variances $D$, since different researchers may care about different parameters. For instance, it is sufficient to identify $C$ for researchers interested in learning causal relations among latent variables~\citep{silva2006bayesian,silva2006learning,cui2016copula}, with no need to worry about additional conditions to identify $\Lambda$ and $D$. Under sufficient identification conditions, we proved that our approach is consistent for MCAR data and empirically showed that it works quite well for MAR data.

In the experiments, our approach outperforms DWLS even under the assumptions of DWLS. Apparently, the approximations inherent in DWLS, such as the use of the polychoric correlation and its asymptotic covariance, incur a small loss in accuracy compared to an integral approach like the BGCF. When the data follow from a more complicated distribution and contain missing values, the advantage of BGCF over its competitors becomes more prominent. Another highlight of our approach is that the Gibbs sampler converges quite fast, where the burn-in period is rather short. To further reduce the time complexity, a potential optimization of the sampling process is available~\citep{kalaitzis2013flexible}.

There are various generalizations to our inference approach. While our focus in this paper is on the correlated $k$-factor models, it is straightforward to extent the current procedure to other class of latent models that are often considered in CFA, such as bi-factor models and second-order models, by simply adjusting the sparsity structure of the prior graph $G$. Also, one may concern models with impure measurement indicators, e.g., a model with an indicator measuring multiple factors or a model with residual covariances~\citep{bollenstructural}, which can be easily solved with BGCF by changing the sparsity pattern of $\Lambda$ and $D$. Another line of future work is to analyze standard errors and confidence intervals while this paper concentrates on the accuracy of parameter estimates. Our conjecture is that BGCF is still favorable because it naturally transfers the extra variability incurred by missing values to the posterior Gibbs samples: we indeed observed a growing variance of the posterior distribution with the increase of missing values in our simulations. On top of the posterior distribution, one could conduct further studies, e.g., causal discovery over latent factors~\citep{silva2006learning,cui2018learning}, regression analysis (as we did in Section~\ref{sec:application}), or other machine learning tasks.

\section*{Appendix A: Proof of Theorem 1}
\setcounter{thm}{0}
\begin{thm}[Consistency of the BGCF Approach]
	\label{thm:consistency} 		
	Let $\bm{Y}_n=(\bm{y}_1,\ldots,\bm{y}_n)^T$ be independent observations drawn from a Gaussian copula factor model. If $\bm{Y}_n$ is complete (no missing data) or contains missing values that are missing completely at random, then
	\begin{gather*}
	\lim\limits_{n \to\infty} P\big(\hat{C}_n = C_0\big) = 1 \:, \\
	\lim\limits_{n \to\infty} P\big(\hat{\Lambda}_n = \Lambda_0\big) = 1 \:, \\
	\lim\limits_{n \to\infty} P\big(\hat{D}_n = D_0\big) = 1 \:,
	\end{gather*}
	where $\hat{C}_n$, $\hat{\Lambda}_n$, and $\hat{D}_n$ are parameters learned by BGCF, while $C_0$, $\Lambda_0$, and $D_0$ are the true ones.
\end{thm}

\begin{proof}
	If $S = \Lambda C \Lambda^T + D$ is the response vector's covariance matrix, then its correlation matrix is $\widetilde{S} = V^{-\frac{1}{2}} S V^{-\frac{1}{2}} = V^{-\frac{1}{2}} \Lambda C \Lambda^T V^{-\frac{1}{2}} + V^{-\frac{1}{2}} D V^{-\frac{1}{2}} = \widetilde{\Lambda} C \widetilde{\Lambda}^T + \widetilde{D}$, where $V$ is a diagonal matrix containing the diagonal entries of $S$. We make use of Theorem 1 from~\citet{murray2013bayesian} to show the consistency of $\widetilde{S}$. Our factor-analytic prior puts positive probability density almost everywhere on the set of correlation matrices that have a $k$-factor decomposition. Then, by applying Theorem 1 in~\citet{murray2013bayesian}, we obtain the consistency of the posterior distribution on the response vector's correlation matrix for complete data, i.e.,
		\begin{equation}
		\label{eq:consistency_S}
		\lim_{n \rightarrow \infty} \Pi(\widetilde{S} \in \mathcal{V}(\widetilde{S}_0) | \bm{Z}_n \in \D(\bm{Y}_n) ) = 1 \; a.s. \; \forall \; \mathcal{V}(\widetilde{S}_0),
		\end{equation}
	where $\D(\bm{Y}_n)$ is the space restricted by observed data, and $\mathcal{V}(\widetilde{S}_0)$ is a neighborhood of the true parameter $\widetilde{S}_0$. When the data contain missing values that are completely at random (MCAR), we can also directly obtain the consistency of $\widetilde{S}$ by again using Theorem 1 in~\citet{murray2013bayesian}, with an additional observation that the estimation of ordinary and polychoric/polyserial correlations from pairwise complete data is still consistent under MCAR. That is to say, the consistency shown in Equation~(\ref{eq:consistency_S}) also holds for data with MCAR missing values.
	
	From this point on, to simplify notation, we will omit adding the tilde to refer to the rescaled matrices $\widetilde{S}$, $\widetilde{\Lambda}$, and $\widetilde{D}$. 
	Thus, $S$ from now on refers to the correlation matrix of the response vector. $\Lambda$ and $D$ refer to the scaled factor loadings and noise variance respectively.
	
	The Gibbs sampler underlying the BGCF approach has the posterior of $\Sigma$ (the correlation matrix of the integrated vector $\bm{X}$) as its stationary distribution. $\Sigma$ contains  
	$S$, the correlation matrix of the response random vector, in the upper left block and $C$ in the lower right block. Here $C$ is the correlation matrix of factors, which implicitly depends on the \textit{Gaussian copula factor model} from Definition 1 of the main paper via the formula $S = \Lambda C \Lambda^T + D$. In order to render this decomposition identifiable, we need to put constraints on $C$, $\Lambda$, $D$. Otherwise, we can always replace $\Lambda$ with $\Lambda U$ and $C$ with $U^{-1} C U^{-1}$, where $U$ is any $k \times k$ invertible matrix, to obtain the equivalent decomposition $S = (\Lambda U)(U^{-1} C U^{-T})(U^T \Lambda ^T) + D$. However, we have assumed that $\Lambda$ follows a particular sparsity structure in which there is only a single non-zero entry for each row. This assumption restricts the space of equivalent solutions, since any $\Lambda U$ has to follow the same sparsity structure as $\Lambda$. More explicitly, $\Lambda U$ maintains the same sparsity pattern if and only if $U$ is a diagonal matrix (Lemma~\ref{lemm:lambda}).
	
	By decomposing $S$, we get a class of solutions for $C$ and $\Lambda$, i.e., $U^{-1} C U^{-1}$ and $\Lambda U$,  where $U$ can be any invertible diagonal matrix. In order to get a unique solution for $C$, we impose two identifying conditions: 1) we restrict $C$ to be a correlation matrix; 2) we force the first non-zero entry in each column of $\Lambda$ to be positive. These conditions are sufficient for identifying $C$ uniquely (Lemma~\ref{lemm:identifiability}). We point out that these sufficient conditions are not unique. For example, one could replace the two conditions with restricting the first non-zero entry in each column of $\Lambda$ to be one. The reason for our choice of conditions is to keep it consistent with our model definition where $C$ is a correlation matrix. Under the two conditions for identifying $C$, factor loadings $\Lambda$ and residual variances $D$ are also identified except for the case in which there exists one factor that is independent of all the others and this factor only has two indicators. For such a factor, we have 4 free parameters (2 loadings, 2 residuals) while we only have 3 available equations (2 variances, 1 covariance), which yields an underdetermined system. Therefore, the identifiability of $\Lambda$ and $D$ relies on the observation that a factor has a single or at least three indicators if it is independent of all the others. See Lemmas~\ref{lemm:identifiability_lambda} and~\ref{lemm:identifiability_D} for detailed analysis.
	
	Now, given the consistency of $S$ and the unique smooth map from $S$ to $C$, $\Lambda$, and $D$, we obtain the consistency of the posterior mean of the parameter $C$, $\Lambda$, and $D$, which concludes our proof.

\end{proof}

%
%

\begin{lemm}
	\label{lemm:lambda}
	If $\Lambda = (\lambda_{ij})$ is a $p \times k$ factor loading matrix with only a single non-zero entry for each row, then $\Lambda U$ will have the same sparsity pattern if and only if $U = (u_{ij})$ is diagonal.
\end{lemm}
	\begin{proof}
		($\Rightarrow$) We prove the direct statement by contradiction. We assume that $U$ has an off-diagonal entry that is not equal to zero. We arbitrarily choose that entry to be $u_{rs}, r, s \in \{1, 2, \ldots, k\}, r \neq s$. Due to the particular sparsity pattern we have chosen for $\Lambda$, there exists $q \in \{1, 2, \ldots, p\}$ such that $\lambda_{qr} \neq 0$ and $\lambda_{qs} = 0$, i.e., the unique factor corresponding to the response $Z_q$ is $\eta_r$. However, we have $(\Lambda U)_{qs} = \lambda_{qr} u_{rs} \neq 0$, which means $(\Lambda U)$ has a different sparsity pattern from $\Lambda$. We have reached a contradiction, therefore $U$ is diagonal.
		
		($\Leftarrow$) If $U$ is diagonal, i.e., $U = \diag(u_1, u_2, \ldots, u_k)$, then $(\Lambda U)_{ij} = \lambda_{ij} u_j$. This means that $(\Lambda U)_{ij} = 0 \iff \lambda_{ij} u_j = 0 \iff \lambda_{ij} = 0$, so the sparsity pattern is preserved.
	\end{proof}

\begin{lemm}[Identifiability of $C$] \label{lemm:identifiability}
	Given the factor structure defined in Section 3 of the main paper, we can uniquely recover $C$ from $S = \Lambda C \Lambda^T + D$ if 1) we constrain $C$ to be a correlation matrix; 2) we force the first element in each column of $\Lambda$ to be positive.
\end{lemm}
	\begin{proof}
		Here we assume that the model has the stated factor structure, i.e., that there is some $\Lambda$, $C$, and $D$ such that $S = \Lambda C \Lambda^T + D$. We then show that our chosen restrictions are sufficient for identification using an argument similar to that in~\citet{anderson1956statistical}. 
		
		The decomposition $S = \Lambda C \Lambda^T + D$ constitutes a system of $\frac{p (p + 1)}{2}$ equations:
		\begin{equation} \label{eqn:decomposition}
		\begin{aligned}
		s_{ii} & = \lambda_{if(i)}^2 + d_{ii} \\
		s_{ij} & = c_{f(i)f(j)} \lambda_{if(i)} \lambda_{jf(j)} \: , \: i < j \:, \\
		\end{aligned}
		\end{equation} %
		
		where $S = (s_{ij}), \Lambda = (\lambda_{ij}), C = (c_{ij}), D = (d_{ij})$, and $f : \{1, 2, \ldots, p \} \to \{1, 2, \ldots, k \} $ is the map from a response variable to its corresponding factor. Looking at the equation system in~\eqref{eqn:decomposition}, we notice that each factor correlation term $c_{qr}, q \neq r$, appears only in the equations corresponding to response variables indexed by $i$ and $j$ such that $f(i) = q$ and $f(j) = r$ or vice versa. This suggests that we can restrict our analysis to submodels that include only two factors by considering the submatrices of $S, \Lambda, C, D$ that only involve those two factors. To be more precise, the idea is to look only at the equations corresponding to the submatrix $S_{f^{-1}(q) f^{-1}(r)}$, where $f^{-1}$ is the preimage of $\{1, 2, \ldots, k \}$ under $f$. Indeed, we will show that we can identify each individual correlation term corresponding to pairs of factors only by looking at these submatrices. Any information concerning the correlation term  provided by the other equations is then redundant.
		
		
		
		Let us then consider an arbitrary pair of factors in our model and the corresponding submatrices of $\Lambda$, $C$, $D$, and $S$. (The case of a single factor is trivial.) In order to simplify notation, we will also use $\Lambda$, $C$, $D$, and $S$ to refer to these submatrices. We also re-index the two factors involved to $\eta_1$ and $\eta_2$ for simplicity. In order to recover the correlation between a pair of factors from $S$, we have to analyze three separate cases to cover all the bases (see  Figure~\ref{fig:GaussianCopulaModelDemo} for examples concerning each case): %
		\begin{enumerate}
			\item The two factors are not correlated, i.e., $c_{12} = 0$. (There are no restrictions on the number of response variables that the factors can have.)
			\item The two factors are correlated, i.e., $c_{12} \neq 0$, and each has a single response, which implies that $Z_1 = \eta_1$ and $Z_2 = \eta_2$.
			\item The two factors are correlated, i.e., $c_{12} \neq 0$, but at least one of them has at least two responses.
		\end{enumerate}
		
		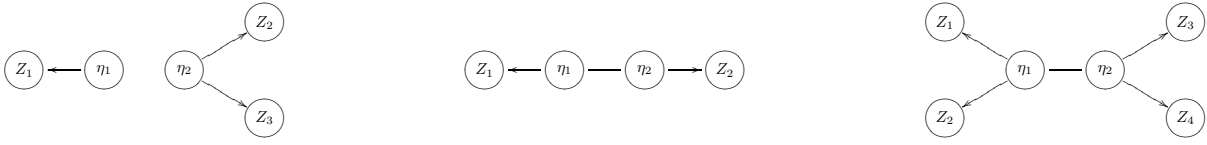
\begin{figure*}[!h]
			\hspace{10pt}
			\begin{tabular}{ccc}
				\parbox[t]{0.33\textwidth}{%
					\scalebox{0.6}{\centerline{\xymatrix @R=1em @C=2em{
								& & & *=<3em,2em>[o][F]{Z_{2}} \\
								*=<3em,2em>[o][F]{Z_1} & *=<3em,2em>[o][F]{\eta_1} \ar[l] &
								*=<3em,2em>[o][F]{\eta_2} \ar[ur] \ar[rd] & \\
								& & & *=<3em,2em>[o][F]{Z_{3}}  \\
				}}}}
				\parbox[t]{0.33\textwidth}{%
					\scalebox{0.6}{\centerline{\xymatrix @R=1em @C=2em{
								& & & *=<3em,2em>[o][B]{} \\
								*=<3em,2em>[o][F]{Z_1} & *=<3em,2em>[o][F]{\eta_1} \ar[l] \ar@{-}[r] &
								*=<3em,2em>[o][F]{\eta_2} \ar[r] &
								*=<3em,2em>[o][F]{Z_{2}} \\
								& & & *=<3em,2em>[o][B]{}  \\
				}}}}
				\parbox[t]{0.33\textwidth}{%
					\scalebox{0.6}{\centerline{\xymatrix @R=1em @C=2em{
								*=<3em,2em>[o][F]{Z_1} & & & *=<3em,2em>[o][F]{Z_{3}} \\
								& *=<3em,2em>[o][F]{\eta_1} \ar[ul] \ar[dl] \ar@{-}[r] &
								*=<3em,2em>[o][F]{\eta_2} \ar[ur] \ar[rd] &
								\\
								*=<3em,2em>[o][F]{Z_{2}} & & & *=<3em,2em>[o][F]{Z_{4}}  \\
				}}}}
			\end{tabular}
			\caption{Left panel: \textit{Case 1} ($c_{12} = 0$); Middle panel: \textit{Case 2} ($c_{12} \neq 0$ and only one response per factor); Right panel: \textit{Case 3} ($c_{12} \neq 0$ and at least one factor has multiple responses).}
			\label{fig:GaussianCopulaModelDemo}
		\end{figure*} 
		
		\textit{Case 1:} If the two factors are not correlated (see the example in the left panel of Figure~\ref{fig:GaussianCopulaModelDemo}), this fact will be reflected in the matrix $S$. More specifically, the off-diagonal blocks in $S$, which correspond to the covariance between the responses of one factor and the responses of the other factor, will be set to zero. If we notice this zero pattern in $S$, we can immediately determine that $c_{12} = 0$.

		\textit{Case 2:} If the two factors are correlated and each factor has a single associated response (see the middle panel of Figure~\ref{fig:GaussianCopulaModelDemo}), the model reduces to a Gaussian Copula model. Then, we directly get $c_{12} = s_{12}$ since we have put the constraints $Z=\eta$ if $\eta$ has a single indicator $Z$. 
		
		\textit{Case 3:} If at least one of the factors (w.l.o.g., $\eta_{1}$) is allowed to have more than one response (see the example in the right panel of Figure~\ref{fig:GaussianCopulaModelDemo}), we arbitrarily choose two of these responses. We also require one response variable corresponding to the other factor ($\eta_{2}$). We use $\lambda_{i1}, \lambda_{j1}$, and $\lambda_{l2}$ to denote the loadings of these response variables, where $i, j, l \in \{1, 2, \ldots, p \}$. From Equation \eqref{eqn:decomposition} we have:
		\begin{align*}
		s_{ij} & = \lambda_{i1} \lambda_{j1} \\
		s_{il} & = c_{12} \lambda_{i1} \lambda_{l2} \\
		s_{jl} & = c_{12} \lambda_{j1} \lambda_{l2} \:. 
		\end{align*} %
		Since we are in the case in which $c_{12} \neq 0$, which automatically implies that $s_{jl} \neq 0$, we can divide the last two equations to obtain $\frac{s_{il}}{s_{jl}} = \frac{\lambda_{i1}}{\lambda_{j1}}$. We then multiply the result with the first equation to get $\frac{s_{ij} s_{il}}{s_{jl}} = \lambda_{i1}^2$. Without loss of generality, we can say that $\lambda_{i1}$ is the first entry in the first column of $\Lambda$, which means that $\lambda_{i1} > 0$. This means that we have uniquely recovered $\lambda_{i1}$ and $\lambda_{j1}$.
		
		We can also assume without loss of generality that $\lambda_{l2}$ is the first entry in the second column of $\Lambda$, so $\lambda_{l2} > 0$. If $\eta_2$ has at least two responses, we use a similar argument to the one before to uniquely recover $\lambda_{l2}$. We can then use the above equations to get $c_{12}$. If $\eta_2$ has only one response, then $d_{ll} = 0$, which means that $s_{ll} = \lambda_{l2}^2$, so again $\lambda_{l2}$ is uniquely recoverable and we can obtain $c_{12}$ from the equations above.
		
		Thus, we have shown that we can correctly determine $c_{qr}$ only from $S_{f^{-1}(q) f^{-1}(r)}$ in all three cases. By applying this approach to all pairs of factors, we can uniquely recover all pairwise correlations. This means that, given our constraints, we can uniquely identify $C$ from the decomposition of $S$.
		
		%
		%
		%
	\end{proof}

\begin{lemm}[Identifiability of $\Lambda$] \label{lemm:identifiability_lambda}
	Given the factor structure defined in Section~\ref{sec:method} of the main paper, we can uniquely recover $\Lambda$ from $S = \Lambda C \Lambda^T + D$ if 1) we constrain $C$ to be a correlation matrix; 2) we force the first element in each column of $\Lambda$ to be positive; 3) when a factor is independent of all the others, it has either a single or at least three indicators.
\end{lemm}
	\begin{proof}
		Compared to identifying $C$, we need to consider another case in which there is only one factor or there exists one factor that is independent of all the others (the former can be treated as a special case of the latter). When such a factor only has a single indicator, e.g., $\eta_1$ in the left panel of Figure~\ref{fig:GaussianCopulaModelDemo}, we directly identify $d_{11} = 0$ because of the constraint $Z_1 = \eta_1$. When the factor has two indicators, e.g., $\eta_2$ in the left panel of Figure~\ref{fig:GaussianCopulaModelDemo}, we have four free parameters ($\lambda_{22}$, $\lambda_{32}$, $d_{22}$, and $d_{33}$) while we can only construct three equations from $S$ ($s_{22}$, $s_{33}$, and $s_{23}$), which cannot give us a unique solution. Now we turn to the three-indicator case, as shown in Figure~\ref{fig:Demo2}. From Equation \eqref{eqn:decomposition} we have:
		\begin{align*}
		s_{12} & = \lambda_{11} \lambda_{21} \\
		s_{13} & =\lambda_{11} \lambda_{31} \\
		s_{23} & = \lambda_{21} \lambda_{31} \:.
		\end{align*} %
		We then have $\frac{s_{12}s_{13}}{s_{23}} = \lambda_{11}^2$, which has a unique solution for $\lambda_{11}$ together with the second constraint $\lambda_{11}>0$, after which we can naturally get the solutions to $\lambda_{21}$ and $\lambda_{31}$. For the other cases, the proof follows the same line of reasoning as Lemma~\ref{lemm:identifiability}.
		\begin{figure}[h]
			\centering
			\begin{tabular}{c}
				\parbox[t]{0.5\textwidth}{%
					\scalebox{0.7}{\centerline{\xymatrix @R=1em @C=2em{
								& *=<3em,2em>[o][F]{Z_{1}} \\
								*=<3em,2em>[o][F]{\eta_1} \ar[ur] \ar[r] \ar[rd] & *=<3em,2em>[o][F]{Z_{2}} \\
								& *=<3em,2em>[o][F]{Z_{3}}  \\
				}}}}
			\end{tabular}
			\caption{A factor model with three indicators.}
			\label{fig:Demo2}
		\end{figure}
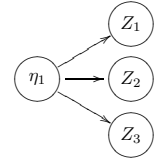 
	\end{proof}

\begin{lemm}[Identifiability of $D$] \label{lemm:identifiability_D}
	Given the factor structure defined in Section~\ref{sec:method} of the main paper, we can uniquely recover $D$ from $S = \Lambda C \Lambda^T + D$ if 1) we constrain $C$ to be a correlation matrix; 2) when a factor is independent of all the others, it has either a single or at least three indicators.
\end{lemm}
	\begin{proof}
		We conduct our analysis case by case. For the case where a factor has a single indicator, we trivially set $d_{ii} = 0$. For the case in Figure~\ref{fig:Demo2}, it is straightforward to get $d_{11} = s_{11} - \lambda_{11}^2$ from $\frac{s_{12}s_{13}}{s_{23}} = \lambda_{11}^2$ (the same for $d_{22}$ and $d_{33}$). Another case we need to consider is Case 3 in Figure~\ref{fig:GaussianCopulaModelDemo}, where we have $\frac{s_{ij} s_{il}}{s_{jl}} = \lambda_{i1}^2$ (see analysis in Lemma~\ref{lemm:identifiability}), based on which we obtain $d_{ii} = s_{ii} - \lambda_{i1}^2$. By applying this approach to all single factors or pairs of factors, we can uniquely recover all elements of $D$.
	\end{proof}

\section*{Appendix B: Extended Simulations}

This section continues the experiments in Section~\ref{sec:simulation} of the main paper, in order to check the influence of the number of categories for ordinal data and the extent of non-normality for nonparanormal data.

\begin{figure}[t]
	\centering	
	\subfloat[Interfactor Correlations]{\includegraphics[scale=0.7]{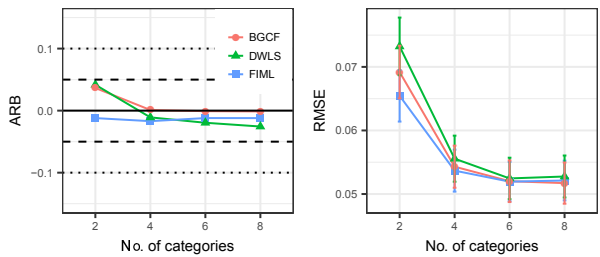}}\label{fig:corr_cat}
	\hfill
	\subfloat[Factor Loadings]{\includegraphics[scale=0.7]{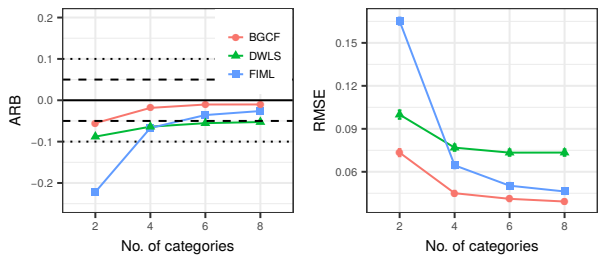}}\label{fig:loadings_cat}
	\caption{Results for $n = 500$ and $\beta = 10\%$ obtained by BGCF, DWLS with PD, and FIML on ordinal data with different numbers of categories, showing the mean of ARB (left panel) and the mean of RMSE with 95\% confidence interval (right panel) over 100 experiments for (a) interfactor correlations and (b) factor loadings, where dashed lines and dotted lines in left panels denote $\pm 5\%$ and $\pm 10\%$ bias respectively.}
	\label{fig:num_cat}
\end{figure}

\subsection*{B1: Ordinal Data with Different Numbers of Categories}

In this subsection, we consider ordinal data with various numbers of categories $c \in \{2, 4, 6, 8\}$, in which the sample size and missing values percentage are set to $n = 500$ and $\beta = 10\%$ respectively. Figure~\ref{fig:num_cat} shows the results obtained by BGCF (Bayesian Gaussian copula factor), DWLS (diagonally weighted least squares) with PD (pairwise deletion), and FIML (full information maximum likelihood), providing the mean of ARB (average relative bias) and the mean of RMSE (root mean squared error) with 95\% confidence interval over 100 experiments for (a) interfactor correlations and (b) factor loadings. In the case of two categories, FIML underestimates factor loadings dramatically, DWLS obtains a moderate bias, while BGCF just gives trivial bias. With an increasing number of categories, FIML gets closer and closer to BGCF, but still BGCF is favorable. 

\subsection*{B2: Nonparanormal Data with Different Extents of Non-normality}

In this subsection, we consider nonparanormal data, in which we use the degrees of freedom $df$ of a $\chi^2$-distribution to control the extent of non-normality (see Section 5.1 of the main paper for details). The sample size and missing values percentage are set to $n = 500$ and $\beta = 10\%$ respectively, while the degrees of freedom varies $df \in \{2, 4, 6, 8\}$.

Figure~\ref{fig:nonpara} shows the results obtained by BGCF, DWLS with PD, and FIML, providing the mean of ARB (left panel) and the mean of RMSE with 95\% confidence interval (right panel) over 100 experiments for (a) interfactor correlations and (b) factor loadings. The major conclusion drawn here is that, while a nonparanormal transformation has no effect on our BGCF approach, FIML is quite sensitive to the extent of non-normality, especially for factor loadings. 

\begin{figure}[H]
	\centering	
	\subfloat[Interfactor Correlations]{\includegraphics[scale=0.7]{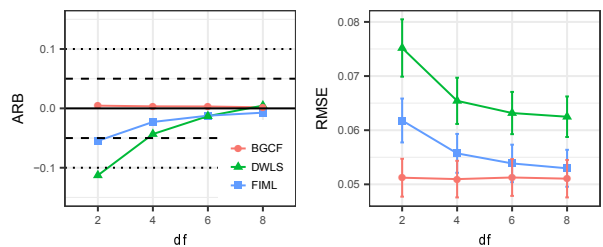}}\label{fig:corr_nonpara}
	\hfill
	\subfloat[Factor Loadings]{\includegraphics[scale=0.7]{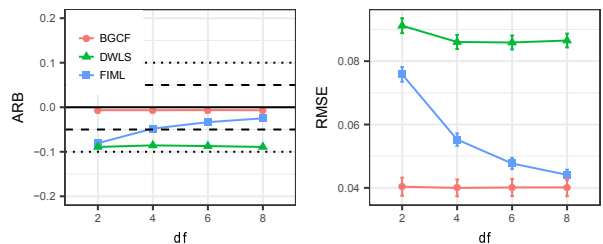}}\label{fig:loadings_nonpara}
	\caption{Results for $n = 500$ and $\beta = 10\%$ obtained by BGCF, DWLS with PD, and FIML on nonparanormal data with different extents of non-normality, for the same experiments as in Figure~\ref{fig:num_cat}.}
	\label{fig:nonpara}
\end{figure}

%


\begin{thebibliography}{42}
	\providecommand{\natexlab}[1]{#1}
	\providecommand{\url}[1]{\texttt{#1}}
	\providecommand{\urlprefix}{URL }
	\expandafter\ifx\csname urlstyle\endcsname\relax
	\providecommand{\doi}[1]{doi:\discretionary{}{}{}#1}\else
	\providecommand{\doi}{doi:\discretionary{}{}{}\begingroup
		\urlstyle{rm}\Url}\fi
	\providecommand{\selectlanguage}[1]{\relax}
	
	\bibitem[{Anderson and Rubin(1956)}]{anderson1956statistical}
	Anderson, T.W., Rubin, H.: Statistical inference in factor analysis.
	\newblock In: Proceedings of the Third Berkeley Symposium on Mathematical
	Statistics and Probability, Volume 5: Contributions to Econometrics,
	Industrial Research, and Psychometry, pp. 111--150. University of California
	Press, Berkeley, Calif. (1956)
	
	\bibitem[{Arbuckle(1996)}]{arbuckle1996full}
	Arbuckle, J.L.: Full information estimation in the presence of incomplete data.
	\newblock Advanced structural equation modeling: Issues and techniques
	\textbf{243}, 277 (1996)
	
	\bibitem[{Barendse et~al.(2015)Barendse, Oort and
		Timmerman}]{barendse2015using}
	Barendse, M., Oort, F., Timmerman, M.: Using exploratory factor analysis to
	determine the dimensionality of discrete responses.
	\newblock Struct. Equ. Modeling. \textbf{22}(1), 87--101 (2015)
	
	\bibitem[{Barnard et~al.(2000)Barnard, McCulloch and
		Meng}]{barnard2000modeling}
	Barnard, J., McCulloch, R., Meng, X.L.: Modeling covariance matrices in terms
	of standard deviations and correlations, with application to shrinkage.
	\newblock Stat. Sinica. pp. 1281--1311 (2000)
	
	\bibitem[{Bollen(1989)}]{bollenstructural}
	Bollen, K.: Structural equations with latent variables.
	\newblock NY Wiley  (1989)
	
	\bibitem[{Browne(1984)}]{browne1984asymptotically}
	Browne, M.W.: Asymptotically distribution-free methods for the analysis of
	covariance structures.
	\newblock Brit. J. Math. Stat. Psy. \textbf{37}(1), 62--83 (1984)
	
	\bibitem[{Byrne(2013)}]{byrne2013structural}
	Byrne, B.M.: Structural equation modeling with EQS: Basic concepts,
	applications, and programming.
	\newblock Routledge (2013)
	
	\bibitem[{Castro et~al.(2015)Castro, Costa, Prates and
		Lachos}]{castro2015likelihood}
	Castro, L.M., Costa, D.R., Prates, M.O., Lachos, V.H.: Likelihood-based
	inference for {T}obit confirmatory factor analysis using the multivariate
	{S}tudent-t distribution.
	\newblock Stat. Comput. \textbf{25}(6), 1163--1183 (2015)
	
	\bibitem[{Cui et~al.(2016)Cui, Groot and Heskes}]{cui2016copula}
	Cui, R., Groot, P., Heskes, T.: Copula {PC} algorithm for causal discovery from
	mixed data.
	\newblock In: Joint European Conference on Machine Learning and Knowledge
	Discovery in Databases, pp. 377--392. Springer (2016)
	
	\bibitem[{Cui et~al.(2018)Cui, Groot and Heskes}]{cui2018learning}
	Cui, R., Groot, P., Heskes, T.: Learning causal structure from mixed data with
	missing values using {G}aussian copula models.
	\newblock Stat. Comput.  (2018)
	
	\bibitem[{Curran et~al.(1996)Curran, West and Finch}]{curran1996robustness}
	Curran, P.J., West, S.G., Finch, J.F.: The robustness of test statistics to
	nonnormality and specification error in confirmatory factor analysis.
	\newblock Psychol. Methods. \textbf{1}(1), 16 (1996)
	
	\bibitem[{DiStefano(2002)}]{distefano2002impact}
	DiStefano, C.: The impact of categorization with confirmatory factor analysis.
	\newblock Struct. Equ. Modeling. \textbf{9}(3), 327--346 (2002)
	
	\bibitem[{Hoff(2007)}]{hoff2007extending}
	Hoff, P.D.: Extending the rank likelihood for semiparametric copula estimation.
	\newblock Ann. Stat. pp. 265--283 (2007)
	
	\bibitem[{Holzinger and Swineford(1939)}]{holzinger1939study}
	Holzinger, K.J., Swineford, F.: A study in factor analysis: the stability of a
	bi-factor solution.
	\newblock Suppl. Educ. Monogr. \textbf{48} (1939)
	
	\bibitem[{Jones et~al.(2005)Jones, Carvalho, Dobra, Hans, Carter and
		West}]{jones2005experiments}
	Jones, B., Carvalho, C., Dobra, A., Hans, C., Carter, C., West, M.: Experiments
	in stochastic computation for high-dimensional graphical models.
	\newblock Stat. Sci. pp. 388--400 (2005)
	
	\bibitem[{J{\"o}reskog(1969)}]{joreskog1969general}
	J{\"o}reskog, K.G.: A general approach to confirmatory maximum likelihood
	factor analysis.
	\newblock Psychometrika \textbf{34}(2), 183--202 (1969)
	
	\bibitem[{J{\"o}reskog(2005)}]{joreskog2005structural}
	J{\"o}reskog, K.G.: Structural equation modeling with ordinal variables using
	{LISREL}.
	\newblock Tech. rep., Technical report, Scientific Software International,
	Inc., Lincolnwood, IL (2005)
	
	\bibitem[{Kalaitzis and Silva(2013)}]{kalaitzis2013flexible}
	Kalaitzis, A., Silva, R.: Flexible sampling of discrete data correlations
	without the marginal distributions.
	\newblock In: Advances in Neural Information Processing Systems, pp. 2517--2525
	(2013)
	
	\bibitem[{Kaplan(2008)}]{kaplan2008structural}
	Kaplan, D.: Structural equation modeling: Foundations and extensions, vol.~10.
	\newblock Sage Publications (2008)
	
	\bibitem[{Kolar and Xing(2012)}]{kolar2012estimating}
	Kolar, M., Xing, E.P.: Estimating sparse precision matrices from data with
	missing values.
	\newblock In: International Conference on Machine Learning (2012)
	
	\bibitem[{Lancaster and Green(2002)}]{lancaster2002latent}
	Lancaster, G., Green, M.: Latent variable techniques for categorical data.
	\newblock Stat. Comput. \textbf{12}(2), 153--161 (2002)
	
	\bibitem[{Li(2016)}]{li2016confirmatory}
	Li, C.H.: Confirmatory factor analysis with ordinal data: Comparing robust
	maximum likelihood and diagonally weighted least squares.
	\newblock Behav. Res. Methods. \textbf{48}(3), 936--949 (2016)
	
	\bibitem[{Little and Rubin(1987)}]{rja1987statistical}
	Little, R.J., Rubin, D.B.: Statistical analysis with missing data (1987)
	
	\bibitem[{Lubke and Muth{\'e}n(2004)}]{lubke2004applying}
	Lubke, G.H., Muth{\'e}n, B.O.: Applying multigroup confirmatory factor models
	for continuous outcomes to likert scale data complicates meaningful group
	comparisons.
	\newblock Struct. Equ. Modeling. \textbf{11}(4), 514--534 (2004)
	
	\bibitem[{Marsh et~al.(1998)Marsh, Hau, Balla and Grayson}]{marsh1998more}
	Marsh, H.W., Hau, K.T., Balla, J.R., Grayson, D.: Is more ever too much? the
	number of indicators per factor in confirmatory factor analysis.
	\newblock Multivar. Behav. Res. \textbf{33}(2), 181--220 (1998)
	
	\bibitem[{Mart{\'\i}nez-Torres(2006)}]{martinez2006procedure}
	Mart{\'\i}nez-Torres, M.R.: A procedure to design a structural and measurement
	model of intellectual capital: an exploratory study.
	\newblock Informa. Manage. \textbf{43}(5), 617--626 (2006)
	
	\bibitem[{Murphy(2007)}]{murphy2007conjugate}
	Murphy, K.P.: Conjugate {B}ayesian analysis of the {G}aussian distribution.
	\newblock def \textbf{1}(2$\sigma$2), 16 (2007)
	
	\bibitem[{Murray et~al.(2013)Murray, Dunson, Carin and
		Lucas}]{murray2013bayesian}
	Murray, J.S., Dunson, D.B., Carin, L., Lucas, J.E.: Bayesian {G}aussian copula
	factor models for mixed data.
	\newblock J. Am. Stat. Assoc. \textbf{108}(502), 656--665 (2013)
	
	\bibitem[{Muth{\'e}n(1984)}]{muthen1984general}
	Muth{\'e}n, B.: A general structural equation model with dichotomous, ordered
	categorical, and continuous latent variable indicators.
	\newblock Psychometrika \textbf{49}(1), 115--132 (1984)
	
	\bibitem[{Muth{\'e}n et~al.(1997)Muth{\'e}n, du~Toit and
		Spisic}]{muthen1997robust}
	Muth{\'e}n, B., du~Toit, S., Spisic, D.: Robust inference using weighted least
	squares and quadratic estimating equations in latent variable modeling with
	categorical and continuous outcomes.
	\newblock Psychometrika  (1997)
	
	\bibitem[{Muth{\'e}n(2010)}]{muthen2010mplus}
	Muth{\'e}n, L.: Mplus user’s guide, ({M}uth{\'e}n \& {M}uth{\'e}n, los
	angeles).
	\newblock Mplus User's Guide,(Muth{\'e}n \& Muth{\'e}n, Los Angeles)  (2010)
	
	\bibitem[{Olsson(1979)}]{olsson1979maximum}
	Olsson, U.: Maximum likelihood estimation of the polychoric correlation
	coefficient.
	\newblock Psychometrika \textbf{44}(4), 443--460 (1979)
	
	\bibitem[{Poon and Wang(2012)}]{poon2012latent}
	Poon, W.Y., Wang, H.B.: Latent variable models with ordinal categorical
	covariates.
	\newblock Stat. Comput. \textbf{22}(5), 1135--1154 (2012)
	
	\bibitem[{Rhemtulla et~al.(2012)Rhemtulla, Brosseau-Liard and
		Savalei}]{rhemtulla2012can}
	Rhemtulla, M., Brosseau-Liard, P.{\'E}., Savalei, V.: When can categorical
	variables be treated as continuous? {A} comparison of robust continuous and
	categorical {SEM} estimation methods under suboptimal conditions.
	\newblock Psychol. Methods. \textbf{17}(3), 354 (2012)
	
	\bibitem[{Rosseel(2012)}]{rosseel2012lavaan}
	Rosseel, Y.: lavaan: An {R} package for structural equation modeling.
	\newblock J. Stat. Softw. \textbf{48}(2), 1--36 (2012)
	
	\bibitem[{Roverato(2002)}]{roverato2002hyper}
	Roverato, A.: Hyper inverse {W}ishart distribution for non-decomposable graphs
	and its application to {B}ayesian inference for {G}aussian graphical models.
	\newblock Scan. J. Stat \textbf{29}(3), 391--411 (2002)
	
	\bibitem[{Rubin(1976)}]{rubin1976inference}
	Rubin, D.B.: Inference and missing data.
	\newblock Biometrika pp. 581--592 (1976)
	
	\bibitem[{Schafer(1997)}]{schafer1997analysis}
	Schafer, J.L.: Analysis of incomplete multivariate data.
	\newblock CRC press (1997)
	
	\bibitem[{Schafer and Graham(2002)}]{schafer2002missing}
	Schafer, J.L., Graham, J.W.: Missing data: our view of the state of the art.
	\newblock Psychol. Methods. \textbf{7}(2), 147 (2002)
	
	\bibitem[{Silva and Scheines(2006)}]{silva2006bayesian}
	Silva, R., Scheines, R.: Bayesian learning of measurement and structural
	models.
	\newblock In: International Conference on Machine Learning, pp. 825--832 (2006)
	
	\bibitem[{Silva et~al.(2006)Silva, Scheines, Glymour and
		Spirtes}]{silva2006learning}
	Silva, R., Scheines, R., Glymour, C., Spirtes, P.: Learning the structure of
	linear latent variable models.
	\newblock J. Mach. Learn. Res. \textbf{7}(Feb), 191--246 (2006)
	
	\bibitem[{Yang-Wallentin et~al.(2010)Yang-Wallentin, J{\"o}reskog and
		Luo}]{yang2010confirmatory}
	Yang-Wallentin, F., J{\"o}reskog, K.G., Luo, H.: Confirmatory factor analysis
	of ordinal variables with misspecified models.
	\newblock Struct. Equ. Modeling. \textbf{17}(3), 392--423 (2010)
	
\end{thebibliography}

%
%

\end{document}